\documentclass{article}
\usepackage{spconf,amsmath,graphicx}
\usepackage{hyperref}
\usepackage{subcaption}
\usepackage{bm}
\usepackage{colortbl}
\usepackage[numbers,sort&compress]{natbib}

\usepackage{multirow}
\usepackage{booktabs}
\usepackage{amsfonts}
\usepackage{amsthm}
\newtheorem{theorem}{Theorem}

\newtheorem{corollary_of_theorem}{Corollary}[theorem]

\newtheorem*{remark}{Remark}

\def\@fnsymbol#1{\ensuremath{\ifcase#1\or *\or \dagger\or \ddagger\or
   \mathsection\or \mathparagraph\or \|\or **\or \dagger\dagger
   \or \ddagger\ddagger \else\@ctrerr\fi}}

\newcommand{\ssymbol}[1]{^{\@fnsymbol{#1}}}

\title{Boosting of Implicit Neural Representation-based Image Denoiser}

\name{Zipei Yan\textsuperscript{1}, Zhengji Liu\textsuperscript{2}, Jizhou Li\textsuperscript{1}\thanks{This work is supported by City University of Hong Kong under grant 9229120.}}
\address{
\textsuperscript{1}School of Data Science, City University of Hong Kong, Hong Kong SAR \\
\textsuperscript{2}School of Optometry, The Hong Kong Polytechnic University, Hong Kong SAR \\ 
}
      
\begin{document}
\ninept
\maketitle
\begin{abstract}
Implicit Neural Representation (INR) has emerged as an effective method for unsupervised image denoising. However, INR models are typically overparameterized; consequently, these models are prone to overfitting during learning, resulting in suboptimal results, even noisy ones. To tackle this problem, we propose a general recipe for regularizing INR models in image denoising. In detail, we propose to iteratively substitute the supervision signal with the mean value derived from both the prediction and supervision signal during the learning process. We theoretically prove that such a simple iterative substitute can gradually enhance the signal-to-noise ratio of the supervision signal, thereby benefiting INR models during the learning process. Our experimental results demonstrate that INR models can be effectively regularized by the proposed approach, relieving overfitting and boosting image denoising performance.
\end{abstract}
\section{Introduction}
\label{sec:intro}

% where $\bm{n} \in \mathbb{R}^{K}$ is an additive noise and typically assumed to be \textit{i.i.d.} zero-mean Gaussian noise, i.e., $\bm{n} \sim \mathcal{N}(0,\sigma^2)$, where $\sigma$ denotes the standard deviation.

Image denoising is a fundamental inverse problem in low-level computer vision. In this setting, we are given a vectorized noisy observation $\bm{y} \in \mathbb{R}^{K}$ combined by the underlying clean signal $\bm{x} \in \mathbb{R}^{K}$ and noise, which is generally formulated as follows:
% \vspace{-7pt}
\begin{equation}
    \bm{y} = \bm{x} + \bm{n}
    \label{eq:noise}
    % \vspace{-7pt}
\end{equation}
where $\bm{n} \in \mathbb{R}^{K}$ is an additive noise and typically assumed to be \textit{i.i.d.} zero-mean Gaussian noise, i.e., $\bm{n} \sim \mathcal{N}(0,\sigma^2)$, where $\sigma$ denotes the standard deviation. Extensive algorithms and models have been developed for image denoising. Before the era of deep learning, filter-based methods~\cite{zhang2008multiresolution,buades2005review,knaus2015dual}, thresholding methods~\cite{chang2000adaptive,blu2007sure,candes2013unbiased} and block-matching-based methods \cite{dabov2007bm3d,magg2013bm4d} are representative works. Nowadays, deep neural networks, such as DnCNN \cite{zhang2017dncnn}, SwinIR \cite{liang2021swinir} and Restormer \cite{syed2022restormer}, etc., achieve remarkable results. In addition to these supervised learning models, self-supervised methods, such as Noise2Self \cite{batson2019n2s}, Noise2Void \cite{alex2019n2v} and Self2Self \cite{quan2020s2s}, etc., provide another alternative to improve the signal-to-noise ratio (SNR), which only utilizes a single noisy image during training.
% filter-based methods~\cite{zhang2008multiresolution,buades2005review,knaus2015dual,chatterjee2009denoising}
% Nowadays, deep neural networks, such as DnCNN \cite{zhang2017dncnn}, IPT \cite{chen2021ipt}, SwinIR \cite{liang2021swinir}, Restormer \cite{syed2022restormer} and KBNet \cite{zhang2023kbnet}, etc.

Recent advancements in image denoising have underscored the potential of Implicit Neural Representations (INR). Noteworthy methods include DIP~\cite{ulyanov2018dip}, SIREN~\cite{sitzmann2020implicit} and WIRE~\cite{saragadam2023wire}, among others. INR offers a novel approach to representing implicitly defined, continuous, and differentiable signals parameterized by neural networks~\cite{sitzmann2020implicit}. Through INR, networks are trained to map specific coordinates or even random noise directly to the target denoised representation, efficiently adapting to fit the noisy observation.

However, INR models are typically overparameterized~\cite{wang2021early}. Consequently, these models are prone to overfitting during training, resulting in suboptimal results, even fitting noisy ones. Existing solutions to this problem primarily pursue the following directions. The first attempts to stop the training once the model begins to fit noise. Early Stopping~\cite{wang2021early} exemplifies this approach, which utilizes the running variance of past denoised results as a metric to determine the optimal training stopping point. Nevertheless, its performance can be inconsistent, and the hyper-parameter optimization becomes particularly intense. The second one integrates regularization to relieve the overfitting problem, such as Total-Variation (TV)~\cite{RUDIN1992259}, Regularization by Denoising (RED)~\cite{yaniv2017red}, etc. For example, DIP can be effectively regularized by TV~\cite{liu2019diptv,cascarano2021combining}, as well as  RED~\cite{mataev2019deepred}, namely DeepRED. Although effective, these methods require additional computation and involve further hyperparameter tuning. Furthermore, the boosting algorithm~\cite{romano2015boosting,chang2018boost,chang2020boost,ma2022boost} presents a holistic strategy for enhancing the denoised outputs. Yet, despite its generality and efficacy, such methods are typically employed as post-processing measures. This divides the process into two distinct phases: primary denoising followed by the boosting stage. This bifurcation naturally demands additional computations and introduces more hyper-parameters to the process. 

In this work, we propose a simple yet effective method to relieve the overfitting tendencies of INR models, thereby boosting their performance. Specifically, our method revolves around iteratively updating the supervision signal with the mean value derived from both the prediction and the original supervision signal throughout the training phase. This iterative substitution approach termed ITS, is shown to enhance the supervision signal's SNR incrementally. As a result, the learning trajectory of INR models becomes more robust. Unlike conventional techniques, our approach introduces nearly zero extra computational overhead and seamlessly integrates into the learning process of INR models. Our experimental results confirm that the proposed method is an effective regularization tool for INR models, addressing overfitting challenges and boosting overall performance. Our contributions are summarized as follows:
% \vspace{-2pt}
\begin{itemize}
    \item We propose a general paradigm for regularizing INR models in image denoising to boost their performance. 
    \item We theoretically prove the proposed ITS approach can progressively improve the supervision signal's SNR.
    \item We evaluate the proposed method when integrated with typical INR models for image denoising, and experimental results demonstrate the effectiveness of our method.
\end{itemize}
% \section{Related Works}
% \subsection{Image Denoising}

% \textbf{Implicit Neural Representation (INR) for image denoising.} 

% INR used multi-layer perceptrons (MLP) instead of convolutional neural networks which learn the mapping from noise to image to establish an implicit , SIREN replaced ReLU with a periodic, WIRE while those method 

% \subsection{Implicit Neural Representation}

% \textbf{Relieve the overfitting problem of INR.}

% \textbf{Regularization Techniques Tailored for INR} One intrinsic challenge with deep neural architectures, especially INRs, is overparamterized. Given their high parameter count, these models can often "memorize" the noise, leading to subpar denoising. 
% \subsection{Regularization}

% \section{Problem Formulation}
% \vspace{-5pt}
\section{Preliminary}
% \vspace{-5pt}
% IND models learn a mapping from a given coordinate or random noise to the target denoised representation by fitting the noisy observation, which follows:
% The learning process of INR can be formulated as follows:
The INR is trained to map from a given coordinate or random noise to the desired denoised representation during the learning process. This can be mathematically formulated as follows:
\begin{equation}
% \vspace{-8pt}
    \bm{\hat{x}} = f_{\theta}(\bm{z}),
\end{equation}
where $f_{\theta}(\cdot)$ denotes the INR model with $\theta$ representing its parameters. The term $\bm z$ is the given coordinate~\cite{sitzmann2020implicit,saragadam2023wire} or random noise~\cite{ulyanov2018dip}.

The overall objective function for optimizing $f_{\theta}(\cdot)$ is formulated as follows:
\begin{equation}
    \mathcal{L} = \Vert \bm{\hat{x}} - \bm{y} \Vert^2 + \lambda \mathcal{R}(\bm{\hat{x}}),
    \label{eq:loss}
\end{equation}
where $\Vert\cdot \Vert$ denotes the $\ell_2$ norm, and $\mathcal{R}(\cdot)$ is the regularization term for introducing explicit priors, such as TV~\cite{RUDIN1992259,liu2019diptv}, RED~\cite{yaniv2017red,mataev2019deepred}, and others.

% Therefore, at some of iterations before overfitting the noise, the INR models generate 
For simplicity, let's consider when $\lambda=0$, i.e., without any explicit priors, the objective function simplifies to $\mathcal{L} = \Vert \bm{\hat{x}} - \bm{y} \Vert^2$. Within this context, the INR models align with the dominant semantic context (represented by low-frequency signals) before they start overfitting to the noise, symbolized by high-frequency signals~\cite{xu2019frequency}. 
%Considering a scenario where $\lambda=0$ (without explicit priors), the objective becomes $\mathcal{L} = \mathbb{E}[\Vert \bm{\hat{x}} - \bm{y} \Vert^2]$. In this case, the INR models initially capture the primary semantic contents (low-frequency signals) and eventually overfit to the noise (high-frequency signals).

% During the learning process, $\Vert \bm{e} \Vert$ goes like an inverse bell curve~\cite{ulyanov2018dip,wang2021early,xu2023revisiting}.

Additionally, when the noise levels increase, the learning becomes more challenging, and INR models are more vulnerable to noise disturbances~\cite{wang2021early,xu2023revisiting}. In contrast, if the noise level is small, the learning is much easier; more aggressively, in an extreme case that $\bm{y} = \bm{x}$ if $\sigma=0$, the objective function leads to fitting the clean ground truth.
%Moreover, it has been observed that as noise levels increase, the training process becomes more challenging, making INR models more susceptible to noise \cite{wang2021early,xu2023revisiting}. However, with reduced noise levels, training is more simplified. In an ideal scenario where $\bm{y} = \bm{x}$ and $\sigma=0$, the objective would be to fit the clean signal which is the ground truth.

% Consider an extreme case that $\bm{y} = \bm{x}$ if $\sigma=0$, i.e., the noisy observation is clean the ground truth, the

% Based on these two observations, we raise a simple question: Can we utilize the internal denoised results to improve the supervision signal: noisy observations?

% Based on these two observations, we propose that the denoised results during the learning process can be utilized to improve the supervision signal: noisy observations.

% Based on these two observations, we raise a simple question: Can we clean up the noisy supervision signal to generate a less noisy supervision signal? If so, the INR model can learn much better because we lower the noise level to provide a better supervision signal; even more extreme, if we clean up the original noisy supervision signal to the clean ground truth, then there is no overfitting problem for noise.

These observations raise a simple question: Can we clean up the noisy supervision signal during the INR training process? If so, the INR model can learn much better because the noise level is reduced to provide a better supervision signal; even more extreme, if we clean up the original noisy supervision signal to the clean ground truth, then the issue of overfitting to noise will be eliminated.

%Given these insights, we pose a fundamental question: Is it possible to refine the noisy supervision signal? A cleaner supervision signal would undoubtedly aid in training the INR model. Ideally, if we can entirely eliminate the noise from the supervision signal, the overfitting issue would no longer be a concern.

\newcommand{\textsupsub}[3]{$\text{#1}^{\text{#2}}_{\text{#3}}$}

\newcommand{\bmtextsupsub}[3]{$\textbf{\text{#1}}^{\text{#2}}_{\text{#3}}$}

\begin{table*}[!t]
    \caption{Experimental results on Set9 and Set11. Metric: Average PSNR / SSIM. The best results are \underline{underlined}, and improved INR results are \textbf{bold-faced}. \textsuperscript{*} and \textsuperscript{**} denote statistical significance in the paired t-test (\textsuperscript{*} indicates $p \leq 0.05$, \textsuperscript{**} indicates $p \leq 0.001$). }
    
    \label{tab:main_results}
    \centering
    \resizebox{\textwidth}{!}{
    \begin{tabular}{lccccccc>{\columncolor[gray]{0.8}}c>{\columncolor[gray]{0.8}}c>{\columncolor[gray]{0.8}}c}
        \toprule
         & & \multicolumn{3}{c}{Baseline} & \multicolumn{3}{c}{INR} & \multicolumn{3}{c}{\cellcolor[gray]{0.8}INR w/ ITS (Ours)} \\
        \cmidrule(r){3-5}\cmidrule(r){6-8}\cmidrule(r){9-11}
        Dataset & $\sigma$ & BM3D & DnCNN & Restormer & DIP & SIREN & WIRE & DIP & SIREN & WIRE \\
        \midrule
        \multirow{4}{*}{Set9} & 25 & 30.15 / 0.80 & 30.82 / 0.84 &  \underline{32.28 / 0.86} & 28.66 / 0.75 & 28.00 / 0.72 & 27.34 / 0.68 & \textbf{28.92}\textsubscript{(+0.26)} / \textbf{0.76}\textsubscript{(+0.01)} & \textbf{28.24}\textsubscript{(+0.24)} / \textbf{0.73}\textsubscript{(+0.01)} & \textbf{27.69}\textsubscript{(+0.35)} / \textbf{0.70}\textsubscript{(+0.02)} \\
         & 50 & 26.60 / 0.70 & 27.41 / 0.76 & \underline{28.07 / 0.77} & 22.63 / 0.46 & 23.23 / 0.51 & 23.66 / 0.54 & \bmtextsupsub{25.95}{*}{(+3.32)} / \bmtextsupsub{0.67}{**}{(+0.21)} & \textbf{25.15}\textsubscript{(+1.92)} / \textbf{0.62}\textsubscript{(+0.11)} & \textbf{24.82}\textsubscript{(+1.16)} / \textbf{0.60}\textsubscript{(+0.06)} \\
         & 75 & 23.69 / 0.63 & 24.22 / 0.67 & \underline{24.78 / 0.70} & 18.43 / 0.27 & 18.52 / 0.28 & 21.33 / 0.46 & \bmtextsupsub{22.38}{**}{(+3.95)} / \bmtextsupsub{0.51}{**}{(+0.24)} & \bmtextsupsub{21.93}{*}{(+3.41)} / \bmtextsupsub{0.51}{*}{(+0.23)} & \textbf{22.57}\textsubscript{(+1.24)} / \textbf{0.55}\textsubscript{(+0.09)} \\
         & 100 & 21.35 / 0.58 & 20.82 / 0.48 & \underline{22.32 / 0.65} & 15.66 / 0.17 & 15.31 / 0.16 & 19.62 / 0.42 & \bmtextsupsub{19.46}{**}{(+3.80)} / \bmtextsupsub{0.36}{**}{(+0.19)} & \bmtextsupsub{19.09}{*}{(+3.78)} / \bmtextsupsub{0.37}{**}{(+0.21)} & \textbf{20.68}\textsubscript{(+1.06)} / \textbf{0.52}\textsubscript{(+0.10)} \\
        \midrule
        \multirow{4}{*}{Set11} & 25 & 29.94 / 0.84 & \underline{29.98 / 0.85} & 27.88 / 0.72 & 25.93 / 0.67 & 25.65 / 0.65 & 26.09 / 0.67 & \textbf{27.10}\textsubscript{(+1.17)} / \textbf{0.75}\textsubscript{(+0.08)} & \textbf{27.11}\textsubscript{(+1.46)} / \textbf{0.72}\textsubscript{(+0.07)} & \textbf{26.38}\textsubscript{(+0.29)} / \textbf{0.70}\textsubscript{(+0.03)} \\
         & 50 & \underline{26.16 / 0.74} & 25.93 / 0.73 & 20.97 / 0.44 & 19.68 / 0.38 & 18.34 / 0.35 & 22.11 / 0.50 & \bmtextsupsub{22.98}{**}{(+3.30)} / \bmtextsupsub{0.54}{*}{(+0.16)} & \bmtextsupsub{22.06}{*}{(+3.72)} / \bmtextsupsub{0.50}{*}{(+0.15)} & \textbf{23.39}\textsubscript{(+1.28)} / \textbf{0.56}\textsubscript{(+0.06)} \\
         & 75 & \underline{23.49 / 0.65} & 20.51 / 0.42 & 17.44 / 0.30 & 15.71 / 0.23 & 14.02 / 0.19 & 18.96 / 0.37 & \bmtextsupsub{19.51}{**}{(+3.80)} / \bmtextsupsub{0.37}{*}{(+0.14)} & \bmtextsupsub{18.08}{*}{(+4.06)} / \bmtextsupsub{0.34}{*}{(+0.15)} & \textbf{20.80}\textsubscript{(+1.84)} / \textbf{0.46}\textsubscript{(+0.09)} \\
         & 100 & \underline{21.43 / 0.58} & 15.16 / 0.20 & 15.69 / 0.23 & 13.16 / 0.15 & 11.49 / 0.12 & 16.81 / 0.29 & \bmtextsupsub{16.92}{**}{(+3.76)} / \bmtextsupsub{0.27}{*}{(+0.11)} & \bmtextsupsub{15.03}{*}{(+3.54)} / \bmtextsupsub{0.22}{*}{(+0.10)} & \textbf{18.79}\textsubscript{(+1.98)} / \textbf{0.38}\textsubscript{(+0.09)} \\
        \bottomrule
    \end{tabular}}
\end{table*}

\newcommand{\dfigwidth}{0.155}

\begin{figure*}[!t]
    \centering
    % denoised
    \begin{subfigure}[b]{\textwidth}
        \centering
        \begin{subfigure}[b]{\dfigwidth\textwidth}
             \centering
             \includegraphics[width=\textwidth]{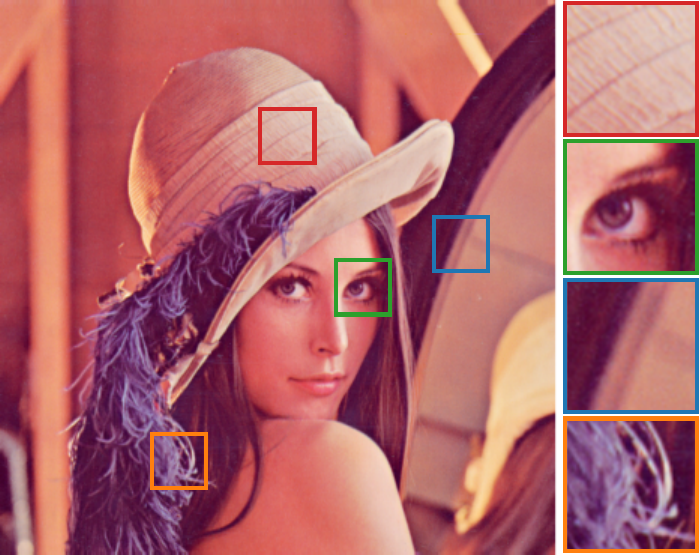}
             \caption*{\scriptsize \textit{Lena}, PSNR / SSIM}
        \end{subfigure}
        \begin{subfigure}[b]{\dfigwidth\textwidth}
            \centering
            \includegraphics[width=\textwidth]{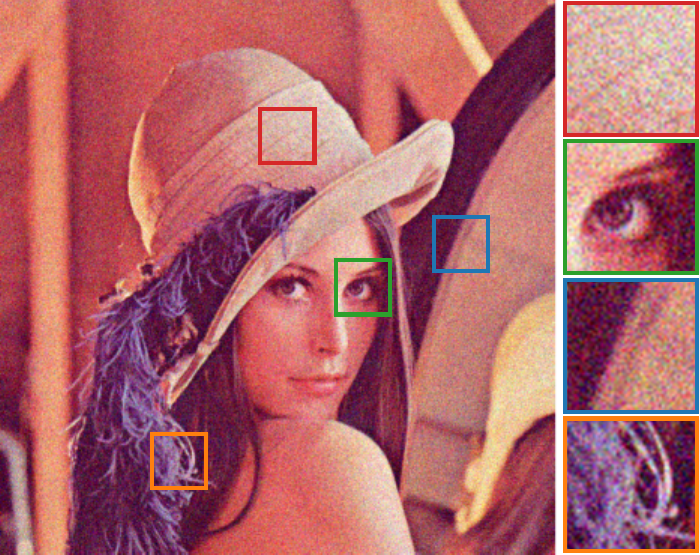}
            \caption*{\scriptsize Noisy($\sigma$=25), 20.37 / 0.30}
        \end{subfigure}
        \begin{subfigure}[b]{\dfigwidth\textwidth}
             \centering
             \includegraphics[width=\textwidth]{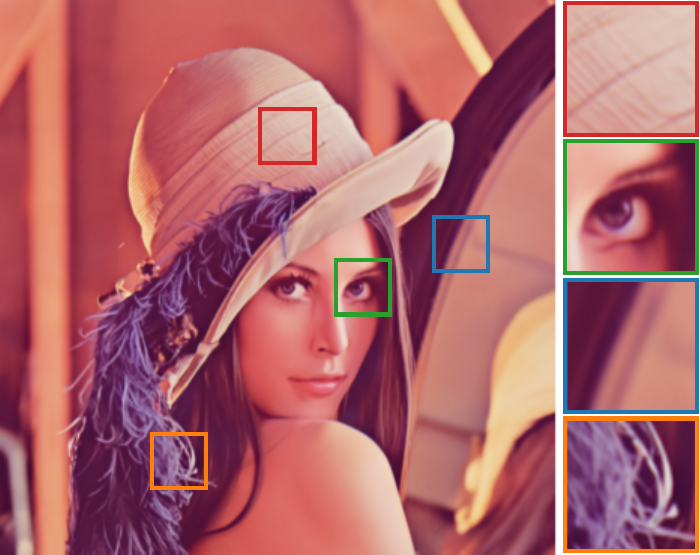}
             \caption*{\scriptsize Restormer, 32.95 / 0.86}
        \end{subfigure}
        \begin{subfigure}[b]{\dfigwidth\textwidth}
            \centering
            \includegraphics[width=\textwidth]{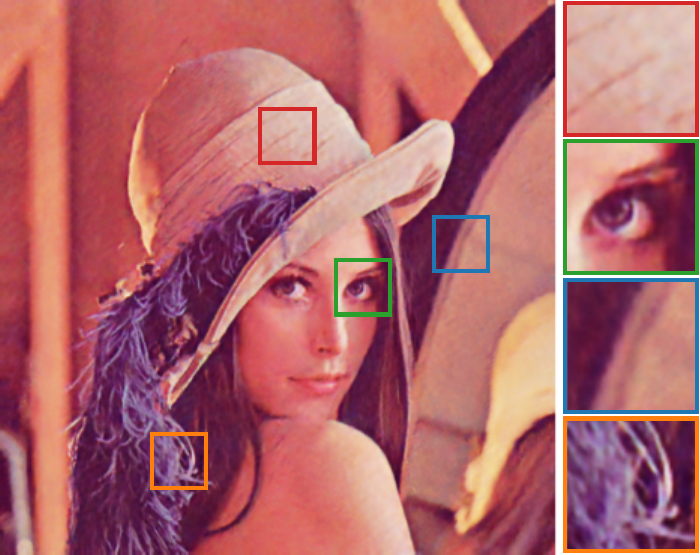}
            \caption*{\scriptsize DIP, 29.73 / 0.78}
        \end{subfigure}
        \begin{subfigure}[b]{\dfigwidth\textwidth}
             \centering
             \includegraphics[width=\textwidth]{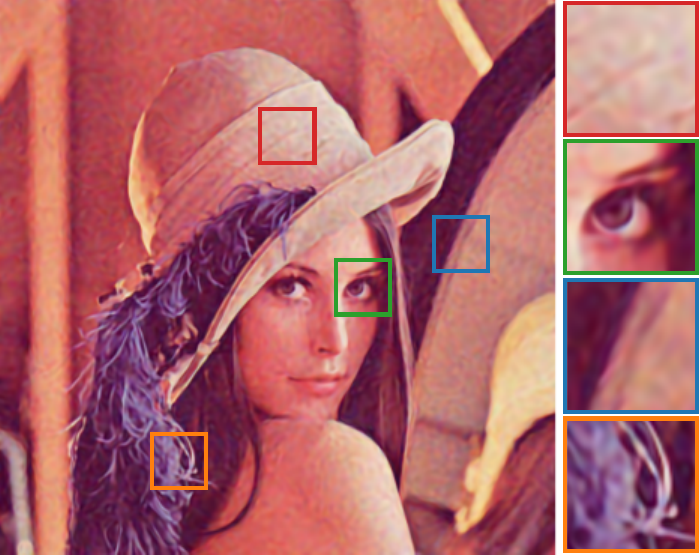}
             \caption*{\scriptsize SIREN, 30.20 / 0.77}
        \end{subfigure}
        \begin{subfigure}[b]{\dfigwidth\textwidth}
             \centering
             \includegraphics[width=\textwidth]{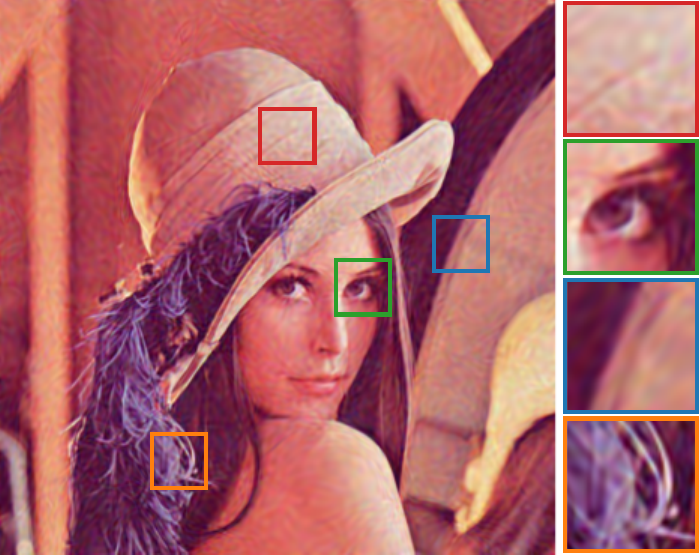}
             \caption*{\scriptsize WIRE, 29.45 / 0.74}
        \end{subfigure}

        \begin{subfigure}[b]{\dfigwidth\textwidth}
            \centering
            \hfill
            \caption*{}
        \end{subfigure}
        \begin{subfigure}[b]{\dfigwidth\textwidth}
             \centering
             \includegraphics[width=\textwidth]{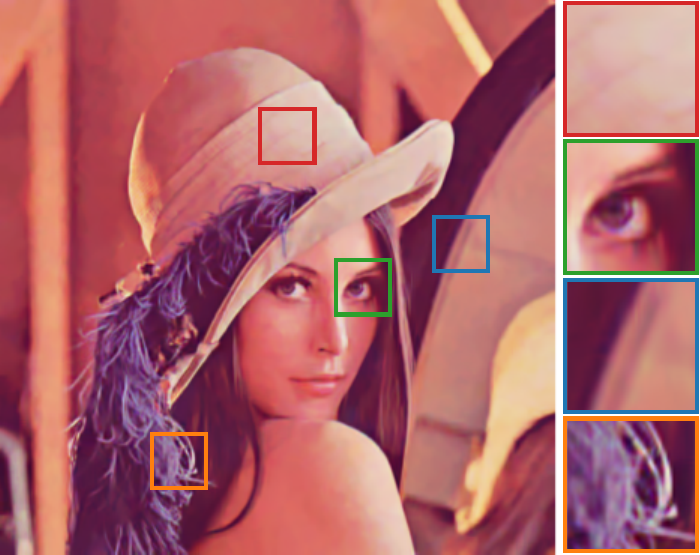}
             \caption*{\scriptsize BM3D, 31.24 / 0.81}
        \end{subfigure}
        \begin{subfigure}[b]{\dfigwidth\textwidth}
             \centering
             \includegraphics[width=\textwidth]{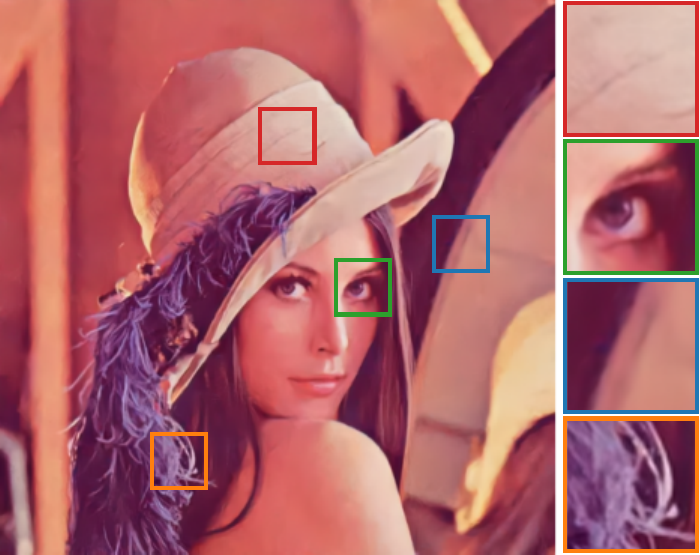}
             \caption*{\scriptsize DnCNN, 31.76 / 0.83}
        \end{subfigure}
        \begin{subfigure}[b]{\dfigwidth\textwidth}
             \centering
             \includegraphics[width=\textwidth]{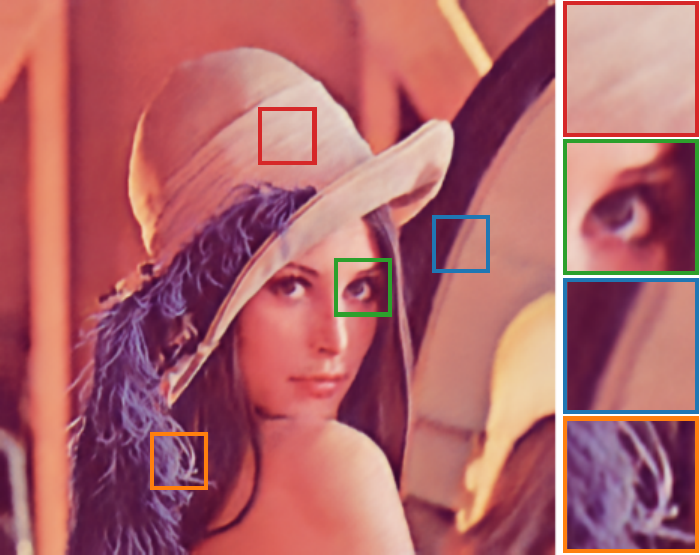}
             \caption*{\scriptsize DIP w/ ITS, 30.48 / 0.81}
        \end{subfigure}
        \begin{subfigure}[b]{\dfigwidth\textwidth}
             \centering
             \includegraphics[width=\textwidth]{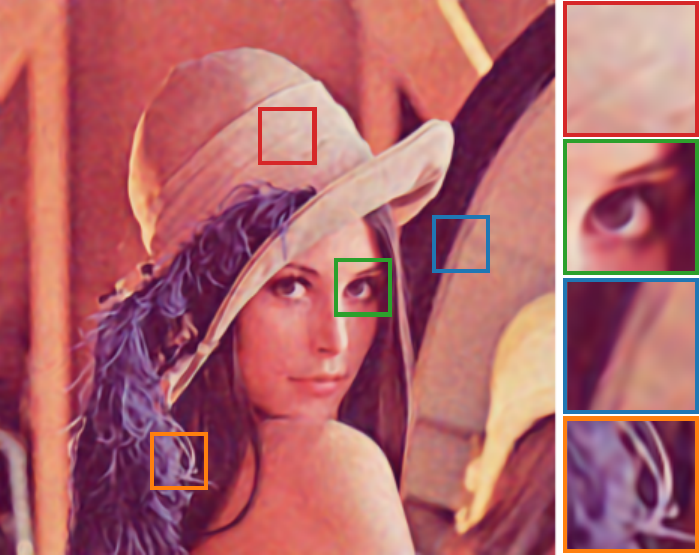}
             \caption*{\scriptsize SIREN w/ ITS, 30.43 / 0.79}
        \end{subfigure}
        \begin{subfigure}[b]{\dfigwidth\textwidth}
             \centering
             \includegraphics[width=\textwidth]{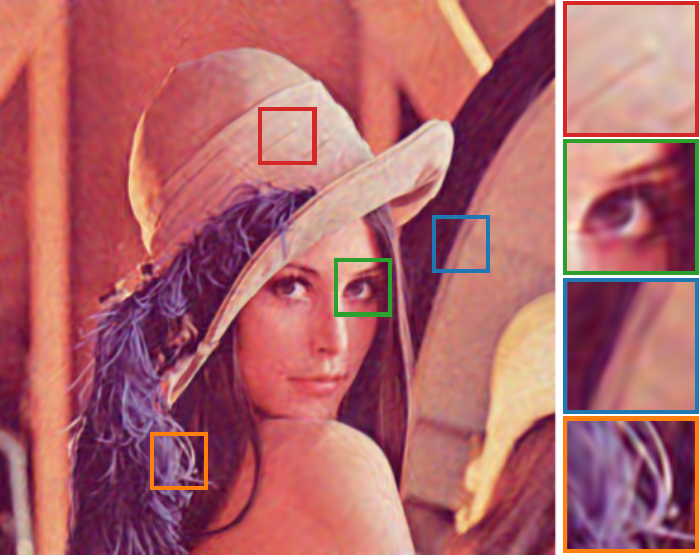}
             \caption*{\scriptsize WIRE w/ ITS, 29.92 / 0.77}
        \end{subfigure}
        % \vspace{-3pt}
        \caption{Denoised results.}
    \end{subfigure}
    % error map
    \begin{subfigure}[b]{\textwidth}
        \centering
        \begin{subfigure}[b]{\dfigwidth\textwidth}
             \centering
             \includegraphics[width=\textwidth]{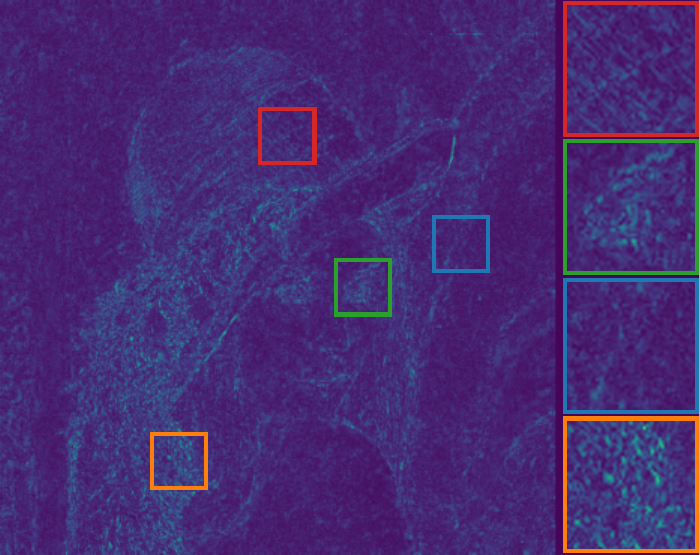}
             \caption*{\scriptsize $\hat{\sigma}(\bm{e}_{\text{DIP}})$:  2.04}
        \end{subfigure}
        \begin{subfigure}[b]{\dfigwidth\textwidth}
             \centering
             \includegraphics[width=\textwidth]{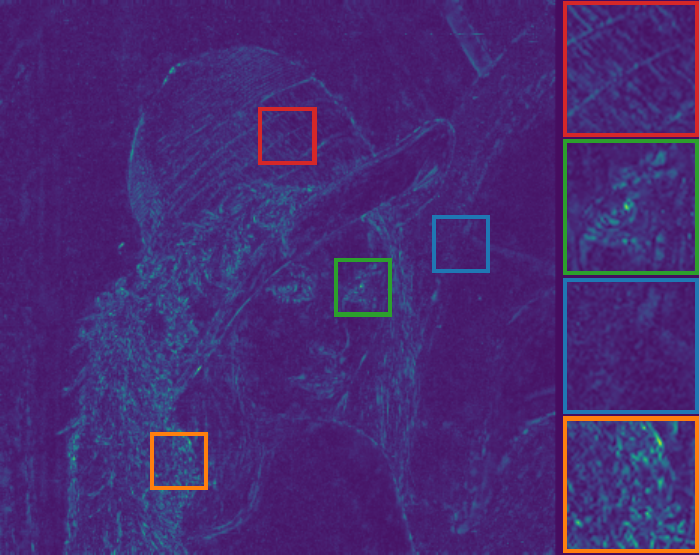}
             \caption*{\scriptsize $\hat{\sigma}(\bm{e}_{\text{DIP w/ ITS}})$:  1.76}
        \end{subfigure}
        \begin{subfigure}[b]{\dfigwidth\textwidth}
             \centering
             \includegraphics[width=\textwidth]{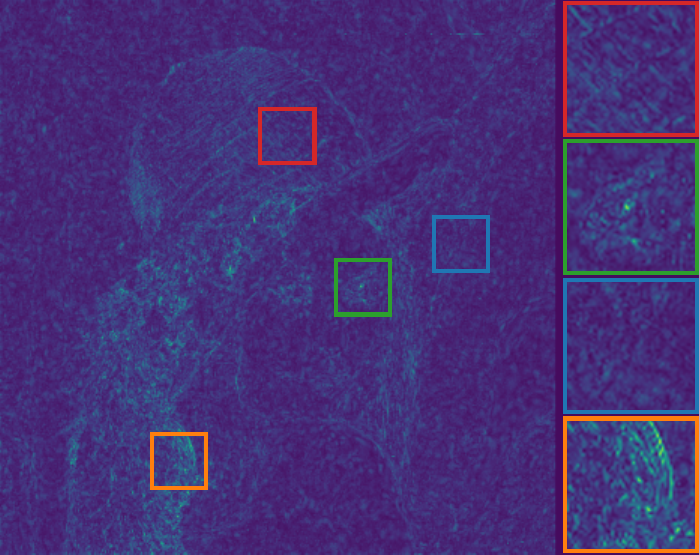}
             \caption*{\scriptsize $\hat{\sigma}(\bm{e}_{\text{SIREN}})$:  1.89}
        \end{subfigure}
        \begin{subfigure}[b]{\dfigwidth\textwidth}
             \centering
             \includegraphics[width=\textwidth]{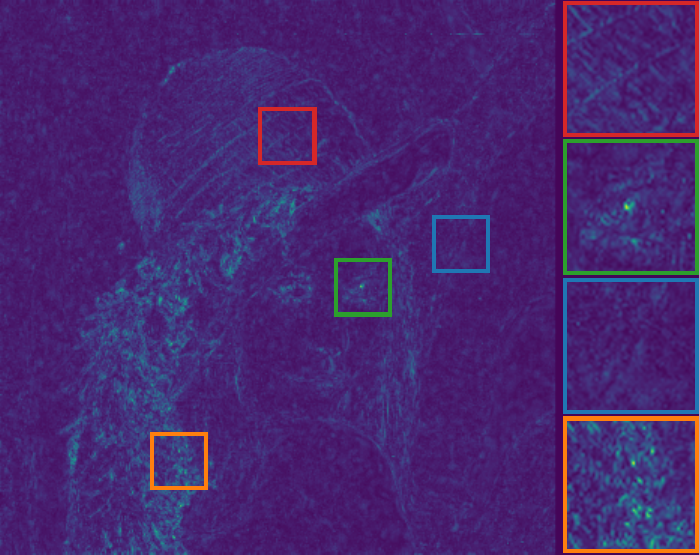}
             \caption*{\scriptsize $\hat{\sigma}(\bm{e}_{\text{SIREN w/ ITS}})$:  1.80}
        \end{subfigure}
        \begin{subfigure}[b]{\dfigwidth\textwidth}
             \centering
             \includegraphics[width=\textwidth]{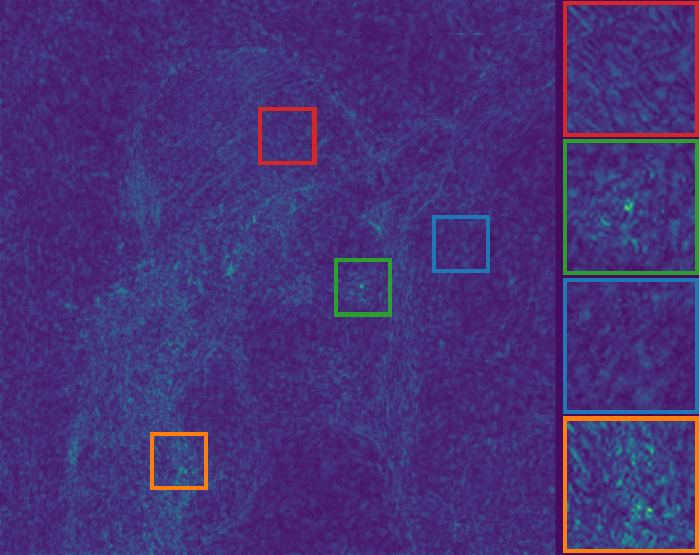}
             \caption*{\scriptsize $\hat{\sigma}(\bm{e}_{\text{WIRE}})$: 2.01}
        \end{subfigure}
        \begin{subfigure}[b]{\dfigwidth\textwidth}
             \centering
             \includegraphics[width=\textwidth]{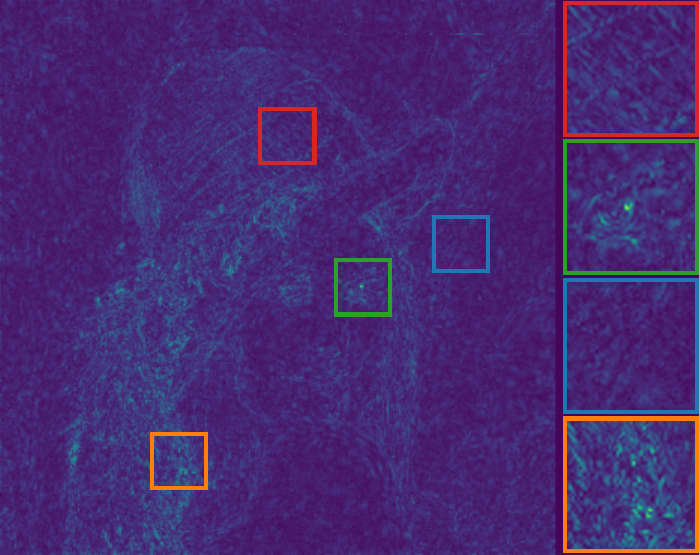}
             \caption*{\scriptsize $\hat{\sigma}(\bm{e}_{\text{WIRE w/ ITS}})$: 1.91}
        \end{subfigure}
        % \vspace{-3pt}
        \caption{Reconstruction error maps.}
    \end{subfigure}
    % \vspace{3pt}
    \caption{Visualization of denoised results on \textit{Lena} with $\sigma$=25. Each denoised result is quantified by PSNR and SSIM. Besides, each reconstruction error map is calculated from Eq.\eqref{eq:denoised}, where $\hat{\sigma}$ indicates the noise level estimated by a robust wavelet-based estimator.}
    \label{fig:denoised}
\end{figure*}

\section{Proposed Algorithm}
\label{sec:alg}

Considering that the INR model is generally optimized using stochastic gradient descent (SGD) within $T$ iterations, then for a given $i$-th iteration, the denoised result $\bm{\hat{x}}^i$ can be expressed as follows~\cite{romano2015boosting,chang2018boost,chang2020boost,jang2021c2n}:
% \begin{equation}
%     \bm{\hat{x}}^i = \bm{x} + \bm{e}^i = \bm{x} + ( \bm{e_n}^i - \bm{e_x}^i),
%     \label{eq:denoised}
% \end{equation}
\begin{equation}
    \begin{aligned}
        \bm{\hat{x}}^i &= \bm{x} + \bm{e}^i \\
        &= \bm{x} + ( \bm{e_n}^i - \bm{e_x}^i),
    \end{aligned}
    \label{eq:denoised}
\end{equation}
where $\bm{e}^i$ denotes reconstruction error. This error can be further decomposed into two parts: $\bm{e_n}^i$ denotes the residual noise from $\bm{n}$ and $\bm{e_x}^i$ stands for the lost components of $\bm{x}$. 

Assuming that at $t$ iteration, the model begins to overfit the noise, then there exists a relationship between $\text{SNR}(\bm{\hat{x}}^i)$ at different iterations, that satisfies:
\begin{equation}
    \text{SNR}(\bm{\hat{x}}^1) < ... < \text{SNR}(\bm{\hat{x}}^t) > ... > \text{SNR}(\bm{\hat{x}}^T),
\end{equation}
where $\text{SNR}(\bm{\hat{x}}^i)=\frac{\Vert \bm{x}^i \Vert}{\Vert \bm{e}^i \Vert}$, and the above relationship's shape mimics a bell curve. In a practical setting, it is infeasible to determine $t$ since only the noisy observation $\bm{y}$ is available without the information of ground truth $\bm{x}$. 

However, given that the noisy observation $\bm{y}$ is described by Eq.\eqref{eq:noise}, we suggest merging the internal denoised result $\bm{\hat{x}}^{i}$ with noisy observation $\bm{y}$ to generate a renewed supervision signal for $(i+1)$-th iteration:
% \vspace{-10pt}
\begin{equation}
\begin{aligned}
    \bm{\hat{y}}^{i+1} &= \frac{\bm{y} + \bm{\hat{x}}^{i}}{2} \\
    &= \bm{x} + \frac{\bm{n} + \bm{e}^{i}}{2},
    \label{eq:replace}
\end{aligned}
% \vspace{-5pt}
\end{equation}
where $\frac{\bm{n} + \bm{e}^{i}}{2}$ is a new reconstruction error. If $\bm{e}^{i} < \bm{n}$ holds true, the above substitution can yield a supervision signal with an enhanced SNR. An illustrative proof is provided in Theorem \ref{thm:improve_snr}.

% \begin{proposition}
%     X
% \end{proposition}

\begin{theorem}
    Assuming the model $f_{\theta}(\cdot)$ is converged by minimizing Eq.\eqref{eq:loss} during the SGD optimization within $T$ iterations, then for any iteration $i$ satisfies $1 \ll i \ll T$, the above substitution from Eq.\eqref{eq:replace} can improve the SNR of the supervision signal, i.e., $\text{SNR}(\bm{\hat{y}}^{i+1}) > \text{SNR}(\bm{y})$.
    \label{thm:improve_snr}
\end{theorem}
% \vspace{-10pt}
\begin{proof}
    For any iteration $i$ ($1 \ll i \ll T$), that model $f_{\theta}(\cdot)$ already (imperfectly/partially) fits visual contents yet does not fully fit the noise, we have: 
    % \vspace{-10pt}
    \begin{align}
        \text{SNR}(\bm{\hat{x}}^{i}) &> \text{SNR}(\bm{y}), \label{eq:snr_cmp} \\
        \frac{\Vert \bm{x} \Vert}{\Vert \bm{e}^i \Vert} &> \frac{\Vert \bm{x} \Vert}{\Vert \bm{n} \Vert}.
        \label{eq:snr_cmp_def}
         % \vspace{-15pt}
    \end{align}
    % where  Eq.\eqref{eq:snr_cmp} 
    Let $\Vert \bm{e}^i \Vert = \delta \Vert \bm{n} \Vert$, Eq.\eqref{eq:snr_cmp_def} implies $\delta < 1$. Considering the SNR of substitution from Eq.\eqref{eq:replace}, we have:
     % \vspace{-5pt}
    \begin{equation}
    \begin{aligned}
        \text{SNR}^2(\bm{\hat{y}}^{i+1}) &= \text{SNR}^2(\frac{\bm{y} + \bm{\hat{x}}^{i}}{2}) \\
        &= \frac{\Vert \bm{x} \Vert^2}{\Vert \frac{\bm{n} + \bm{e}^i}{2} \Vert^2} \\ 
        &\ge \frac{4\Vert \bm{x} \Vert^2}{\Vert \bm{n} \Vert^2 + 2\Vert \bm{n} \Vert \Vert \bm{e}^i \Vert + \Vert \bm{e}^i \Vert^2} \label{eq:ineq} \\
        &= \frac{4\Vert \bm{x} \Vert^2}{(1+\delta)^2\Vert \bm{n} \Vert^2} \\ 
        &= \frac{4}{(1+\delta)^2} \text{SNR}^2(\bm{y}), 
    \end{aligned}
    \end{equation}
     % \vspace{-2pt}
     where Eq.\eqref{eq:ineq} follows the Cauchy-Shwartz inequality.
    \noindent Therefore, we have:
    \begin{align}
\text{SNR}^2(\bm{\hat{y}}^{i+1}) &> \frac{4}{(1+\delta)^2}, \text{SNR}^2(\bm{y}), \\
         \text{SNR}(\bm{\hat{y}}^{i+1}) &> \frac{2}{(1+\delta)} \text{SNR}(\bm{y}).
    \end{align}
    Recall that $\delta < 1$, then we have: $
        \text{SNR}(\bm{\hat{y}}^{i+1}) > \text{SNR}(\bm{y})$.
    Therefore, the substitution can improve the SNR of the supervision signal.
\end{proof}

\begin{remark}
    The worse-case of substitution from Eq.\eqref{eq:replace} is when $\bm{\hat{x}}^{i} = \bm{y}$, i.e., model $f_{\theta}(\cdot)$ perfect fits noisy observation, then $\text{SNR}(\bm{\hat{y}}^{i+1}) = \text{SNR}(\bm{y})$. And the best case is when $\bm{e}^{i}=\bm{0}$, i.e., $\bm{\hat{x}}^{i}=\bm{x}$, which is almost impossible, and no longer need for the substitution.
\end{remark}

% $S = \{iN : i \in \mathbb{Z}, iN \leq T\}$, 
Based on this observation, we propose an iterative substitution (ITS) strategy that performs the updating from Eq.\eqref{eq:replace} with given increasing iterations. For simplicity, we introduce a hyperparameter $N$ to generate a set of iterations $S = \{N, 2N, ..., kN\}$, then for each iteration in $S$, we perform the substitution from Eq.\eqref{eq:replace}, that is:
\begin{equation}
    \bm{\hat{y}}^{kN+1} = \frac{\bm{y} + \bm{\hat{x}}^{kN}}{2}.
    \label{eq:ITS}
\end{equation}

\begin{corollary_of_theorem}
    For a given set of iterations $S = \{N, 2N, ..., kN\}$, assuming $\text{SNR}(\bm{\hat{x}}^{kN}) > ... > \text{SNR}(\bm{\hat{x}}^{2N}) > \text{SNR}(\bm{\hat{x}}^{N}) > \text{SNR}(\bm{y})$, we have $\text{SNR}(\bm{\hat{y}}^{kN+1}) > ... >\text{SNR}(\bm{\hat{y}}^{2N+1}) > \text{SNR}(\bm{\hat{y}}^{N+1}) > \text{SNR}(\bm{\hat{y}}^{i})$, i.e., the SNR of renewed supervision signal in the given set of iterations increases one by one. 
\end{corollary_of_theorem}

The ITS approach capitalizes on the iterative refinement principle during the training process. By progressively substituting the denoised results into the supervision signal, the model constantly refines its understanding of the image and prevents overfitting, thereby enhancing the denoising performance. The increasing iterations ensure that the model doesn't hastily converge but instead gradually refines its output, potentially achieving a balance between noise suppression and preservation of crucial image details.

\section{Experiments}

\subsection{Experimental Setup}

\textbf{Datasets.} We choose two benchmark datasets: Set9 and Set11 \cite{dabov2007bm3d} for experiments. In detail, Set9 consists of nine RGB colorful images, while Set11 comprises eleven grayscale images. 

% \textbf{Noises.} For simulating different AWGNs, we select $\sigma \in \{25, 50, 75, 100\}$.
\textbf{Noises.} We select $\sigma \in \{25, 50, 75, 100\}$ for simulating different noise levels.

\textbf{Baselines.} We choose three baseline methods: BM3D~\cite{dabov2007bm3d}, DnCNN~\cite{zhang2017dncnn}, Restormer~\cite{syed2022restormer} as baselines for comparison.

\textbf{INR models.} We choose three INR models: DIP~\cite{ulyanov2018dip}, SIREN \cite{sitzmann2020implicit}, WIRE~\cite{saragadam2023wire} for experiments. Then, we apply our proposed ITS to them to evaluate the effectiveness of our method.

\textbf{Evaluation method.} For quantitative evaluation, we mainly refer to two metrics: Peak signal-to-noise ratio (PSNR) and structural similarity (SSIM). For qualitative evaluation, we visualize some denoised results.

\textbf{Implementation details.} We utilize the official pre-trained models for baselines: DnCNN and Restormer. Meanwhile, we follow the official implementations for BM3D and INR models: DIP, SIREN, and WIRE.

% For baselines . For INR models xxx.

\subsection{Experimental Results}

\textbf{Quantitative results.} Table \ref{tab:main_results} reports experimental results of average PSNR and SSIM for Set9 and Set11 datasets. Restormer consistently presents the best baseline results of all noise levels in Set9, as the underlined values indicate. DnCNN achieves the best performance when $\sigma=25$ in Set11, while BM3D achieves the best results when $\sigma \in \{50, 75, 100\}$ in Set11. It is worth mentioning that Restormer and DnCNN are pre-trained with large-scale datasets in a supervised manner. Notably, INR models with ITS demonstrate substantial improvement over the standard INR models across various noise levels, as highlighted by boldface. Specifically, with a high noise level, applying ITS in DIP and SIREN results in a great performance enhancement; for example, when $\sigma=75$ on Set9, the PSNR of SIREN increases by 4.06. The enhancements are further validated by paired t-test statistical significance annotations, offering a more generalized perspective on the efficacy of the proposed ITS.

\textbf{Qualitative results.} Fig.~\ref{fig:denoised} visualizes the denoised results from different baselines, as well as INR models and INR models with ITS. To supplement this visual comparison, we further calculate the reconstruction error map from Eq.\eqref{eq:denoised} for each denoised result. According to Eq.\eqref{eq:denoised}, such an error map consists of residual noise and lost components; therefore, we utilize robust wavelet-based estimator \cite{donoho1994ideal} to estimate the (Gaussian) noise standard deviation $\hat{\sigma}$ on it, to reveal the noise level of residual noise. In detail, we observe that Restormer achieves the best denoised result among baselines. Among the INR models, SIREN delivers a better denoising capability. After introducing ITS, all INR models can be effectively regularized, thereby producing better results with increased PSNR and SSIM. In addition, comparing the reconstruction error map from INR models and those from them applied with ITS, we observe that the error maps for INR models incorporating ITS show a notable reduction in residual noise, as highlighted by the visual content difference and reduced $\hat{\sigma}$.

%Opposites with the numerical result, BM3D's result doesn't match its highest score in the quantitative results...?

\subsection{Ablation Study}

In this section, we conduct a series of ablation studies to evaluate the effectiveness of our proposed approach.

% \begin{figure}[t]
%     \centering
%     \begin{subfigure}[b]{\afigwidth\textwidth}
%         \centering
%         \includegraphics[width=\textwidth]{figures/hparam/DIP_psnr.pdf}
%         % \caption*{DIP}
%     \end{subfigure}
%     \begin{subfigure}[b]{\afigwidth\textwidth}
%         \centering
%         \includegraphics[width=\textwidth]{figures/hparam/SIREN_psnr.pdf}
%         % \caption*{SIREN}
%     \end{subfigure}
%     \begin{subfigure}[b]{\afigwidth\textwidth}
%         \centering
%         \includegraphics[width=\textwidth]{figures/hparam/WIRE_psnr.pdf}
%         % \caption*{WIRE}
%     \end{subfigure}
%     \begin{subfigure}[b]{\afigwidth\textwidth}
%         \centering
%         \includegraphics[width=\textwidth]{figures/hparam/DIP_ssim.pdf}
%         \caption*{DIP}
%     \end{subfigure}
%     \begin{subfigure}[b]{\afigwidth\textwidth}
%         \centering
%         \includegraphics[width=\textwidth]{figures/hparam/SIREN_ssim.pdf}
%         \caption*{SIREN}
%     \end{subfigure}
%     \begin{subfigure}[b]{\afigwidth\textwidth}
%         \centering
%         \includegraphics[width=\textwidth]{figures/hparam/WIRE_ssim.pdf}
%         \caption*{WIRE}
%     \end{subfigure}
%     \caption{Caption}
%     \label{fig:ablation}
% \end{figure}

\textbf{Impact of hyper-parameter: $N$.} We study the impact of hyper-parameter $N$ on three INR models with \textit{Lena}. Specifically, we fix $\sigma=25$ and select $N \in \{100, 200, 300, 400\}$. The results are presented in Fig. \ref{fig:ablation}. In general, we observe that ITS can boost the performance of three INR models. Besides, we observe that almost all INR models w/ ITS are not sensitive to the hyper-parameter, except the SIREN. In detail, for SIREN w/ ITS, the PSNRs from different hyper-parameters are close in final iterations, while the SSIMs have a gap of around 0.01.

\textbf{Relieve the overfitting.} As illustrated in Table \ref{tab:main_results} and Fig. \ref{fig:ablation}, the experimental results demonstrate that our method can relieve the overfitting problem of the INR models. The key reason has already been illustrated in Sec. \ref{sec:alg}, where the proposed ITS can iteratively improve the SNR of the supervision signal.

\begin{table}[t]
    \caption{Comparison and integration to regularization: RED and WTV. Experiments on \textit{Lena} with $\sigma=25$.}
    % The best results are \textbf{bold-faced}.
    \label{tab:ablation}
    \centering
    \resizebox{\linewidth}{!}{\begin{tabular}{lccccc}
        \toprule
        \multirow{2}{*}{Method} & Last & Best & Num of Extra & \multirow{2}{*}{Time (min)} \\
         & PSNR / SSIM & PSNR / SSIM & Hparam. & \\
        \midrule
        \multicolumn{5}{l}{\textit{Comparison to regularization methods}} \\
        DIP & 26.74 / 0.60 & 30.18 / 0.79 & 0 & 3.33 \\
        DIP + WTV & 27.55 / 0.74 & 30.29 / 0.80 & 3 & 135.43 \\
        DeepRED & 27.84 / 0.75 & 30.40 / 0.80 & 5 & 26.29 \\
        DIP + ITS & \textbf{30.31} / \textbf{0.79} & \textbf{30.95} / \textbf{0.81} & 1 & 3.41 \\
        \midrule
        \multicolumn{5}{l}{\textit{Integration to regularization methods}} \\
        DIP + WTV + ITS & 29.98 / 0.78 & 30.40 / 0.80 & 3+1 & 137.45 \\
        DeepRED + ITS & \textbf{30.13} / \textbf{0.79} & \textbf{30.87} / \textbf{0.81} & 5+1 & 27.50 \\
        \bottomrule
    \end{tabular}}
\end{table}

% \newcommand{\afigwidth}{0.158}
% \begin{table}[t]
%     \centering
%     \setlength{\tabcolsep}{1.0pt}
%     \renewcommand{\arraystretch}{0.5}
%     % \rotatebox[origin=c]{90}{PSNR}
%     \resizebox{\linewidth}{!}{\begin{tabular}{cccc}
%         \rotatebox{90}{\ \ \ \ \ \ \ \ \ PSNR} & \includegraphics[width=\afigwidth\textwidth]{figures/hparam/DIP_psnr.pdf} & \includegraphics[width=\afigwidth\textwidth]{figures/hparam/SIREN_psnr.pdf} & \includegraphics[width=\afigwidth\textwidth]{figures/hparam/WIRE_psnr.pdf} \\
%         \rotatebox{90}{\ \ \ \ \ \ \ \ \ SSIM} & \includegraphics[width=\afigwidth\textwidth]{figures/hparam/DIP_ssim.pdf} & \includegraphics[width=\afigwidth\textwidth]{figures/hparam/SIREN_ssim.pdf} & \includegraphics[width=\afigwidth\textwidth]{figures/hparam/WIRE_ssim.pdf} \\
%          & DIP & SIREN & WIRE \\
%     \end{tabular}}
%     \captionof{figure}{Impact of hyper-parameter: $N$. Experiments on \textit{Lena} with $\sigma=25$, and select $N \in \{100, 200, 300, 400\}$.}
%     \label{fig:ablation}
% \end{table}

\begin{figure}[t]
    \centering
    \includegraphics[width=0.98\linewidth]{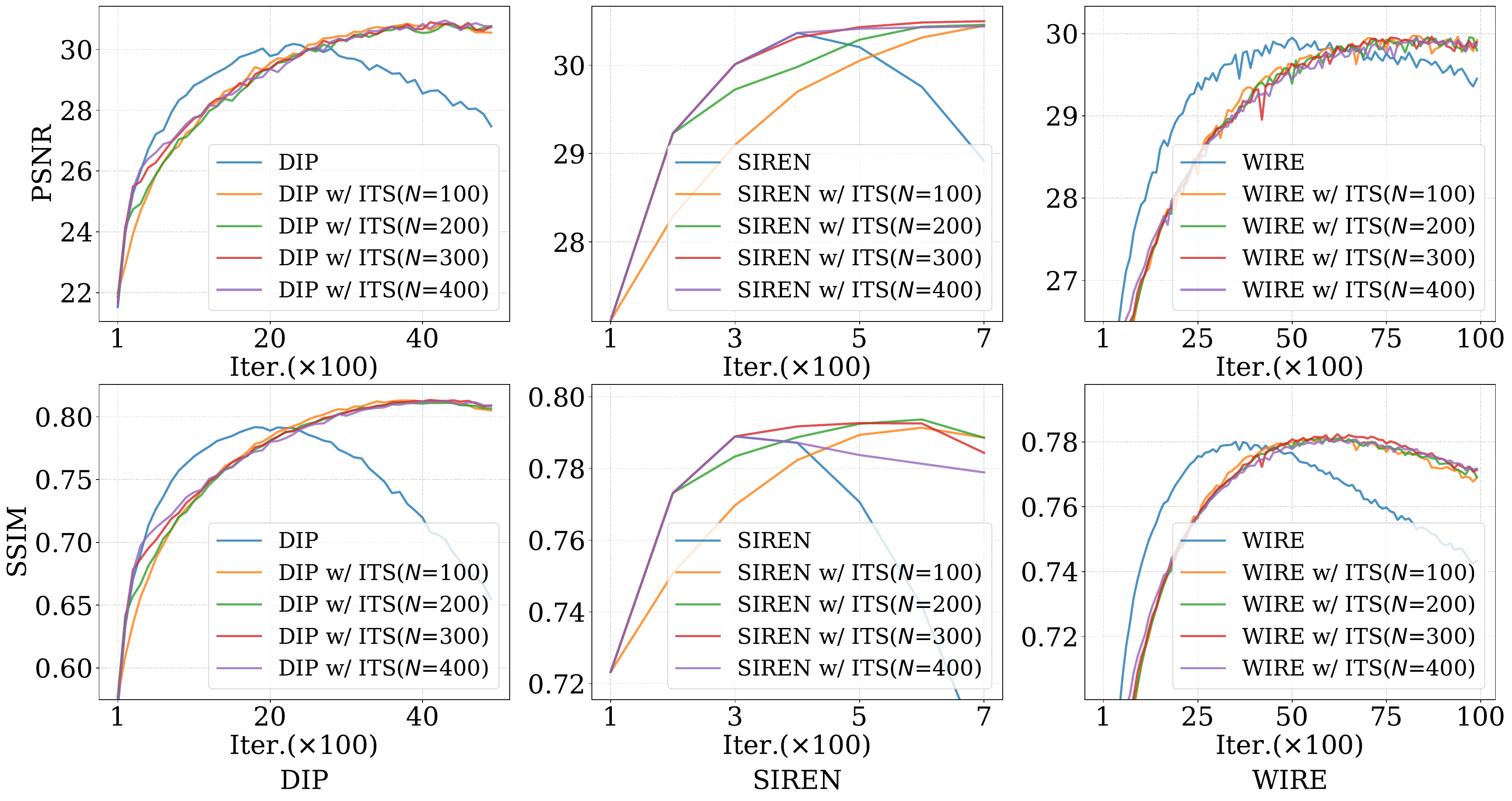}
    \vspace{-0.3cm}
    \caption{Impact of hyper-parameter: $N$. Experiments on \textit{Lena} with $\sigma=25$, and select $N \in \{100, 200, 300, 400\}$.}
    \label{fig:ablation}
\end{figure}

\textbf{Comparison to regularization methods.} We also compare our proposed ITS to regularization methods, including TV and RED. For a fair comparison, we choose the DIP as the baseline model and follow the training configurations from DeepRED \cite{mataev2019deepred}. In detail, we follow DeepRED \cite{mataev2019deepred} for combining DIP with RED. As for DIP with TV, we utilize a weighted version of TV (WTV) \cite{cascarano2021combining}, which has been proven to be better compared to ordinary TV \cite{liu2019diptv}. We conduct all experiments on the \textit{Lena} with $\sigma=25$. The experimental results are reported in Table \ref{tab:ablation}. It is worth reminding that the best performance is impossible to obtain for real-world denoising, because we do not have the ground truth. We report these best performances only to compare the improvement from different methods and to illustrate the overfitting problem. As shown in Table \ref{tab:ablation}, the proposed ITS outperforms two regularizations according to PSNR and SSIM from the last. In addition, from the performance gap between the last and best, we observe that the proposed ITS can consistently benefit the INR model and lead to less overfitting results, which follows our theoretical justification in Sec. \ref{sec:alg}.

\textbf{Integration to regularization methods.} We also study integrating ITS into existing regularization methods. For simplicity, we fix the hyper-parameters of the regularization method and integrate ITS into it. As shown in the last two rows of Table \ref{tab:ablation}, we show that integrating ITS into existing regularization methods also brings better results. 

\textbf{Time consumption.} As shown in the last column of Table \ref{tab:ablation}, we demonstrate that ITS only brings limited extra time consumption because ITS seamlessly integrates into the learning process of INR by only substituting the supervision signals with improved SNR ones. 

\section{Conclusion}
This work proposes a general iterative substitution (ITS) paradigm for regularizing INR models in image denoising. Our theoretical justification proves that ITS can improve the SNR of the supervision signal, thereby benefiting the learning process of INR models. Our experimental results demonstrate the effectiveness and efficiency of the proposed ITS. The major limitation of our method is that we assume the denoised result contains additive reconstruction error as illustrated in Eq.\eqref{eq:denoised}. Although such an assumption is widely adopted in prior works ~\cite{romano2015boosting,chang2018boost,chang2020boost,jang2021c2n}, the general assumption for the inverse problem, i.e., $\bm{\hat{x}}=\bm{A}\bm{x}+\bm{e}$, would be more appropriate, which additionally considers transformation formalized by $\bm{A}$, such as blurring. One future direction is to address the above limitation. In addition, extending our method to video denoising is one future direction. Besides, generalizing our problem and method to the general image inverse problem is another future direction. 

% We evaluate ITS with three INR models, and our experimental results demonstrate that effectiveness of ITS

\vfill\pagebreak

% \section{REFERENCES}
% \label{sec:refs}

% List and number all bibliographical references at the end of the
% paper. The references can be numbered in alphabetic order or in
% order of appearance in the document. When referring to them in
% the text, type the corresponding reference number in square
% brackets as shown at the end of this sentence \cite{C2}. An
% additional final page (the fifth page, in most cases) is
% allowed, but must contain only references to the prior
% literature.

\newpage
\section{REFERENCES}

\bibliographystyle{IEEEbib}
\bibliography{strings,refs}

\end{document}